%% file: main.tex
\documentclass{article}
\usepackage[preprint]{spconf}
\copyrightnotice{\begin{minipage}{\textwidth}\footnotesize\copyright\ 20XX IEEE. Personal use of this material is permitted. Permission from IEEE must be obtained for all other uses, in any current or future media, including reprinting/republishing this material for advertising or promotional purposes, creating new collective works, for resale or redistribution to servers or lists, or reuse of any copyrighted component of this work in other works.\end{minipage}}
\usepackage{amsmath,amssymb,dsfont,graphicx,verbatim}
\graphicspath{{figures/}}
\usepackage{amsmath,amssymb,dsfont,verbatim}
\usepackage{import}
\usepackage{subcaption}
\usepackage{cite}
\usepackage{xcolor}
\usepackage{hyperref}
\usepackage{booktabs} 
\usepackage{tikz}
\usepackage{siunitx}
\usepackage{balance}
\usetikzlibrary{mindmap}

\allowdisplaybreaks[0]

\DeclareMathOperator{\bcw}{{\boldsymbol{\scriptstyle\mathcal{W}}}}
\subimport{custom/}{operators.tex}

\subimport{custom/}{environments.tex}

\title{Graph-Homomorphic Perturbations for Private Decentralized Learning}
\name{Stefan Vlaski and Ali H. Sayed}
\address{School of Engineering, \'{E}cole Polytechnique F\'{e}d\'{e}rale de Lausanne
\thanks{Emails:\{stefan.vlaski, ali.sayed\}@epfl.ch.}}
\begin{document}
%
\maketitle
\begin{abstract}
  Decentralized algorithms for stochastic optimization and learning rely on the diffusion of information as a result of repeated local exchanges of intermediate estimates. Such structures are particularly appealing in situations where agents may be hesitant to share raw data due to privacy concerns. Nevertheless, in the absence of additional privacy-preserving mechanisms, the exchange of local estimates, which are generated based on private data can allow for the inference of the data itself. The most common mechanism for guaranteeing privacy is the addition of perturbations to local estimates before broadcasting. These perturbations are generally chosen independently at every agent, resulting in a significant performance loss. We propose an alternative scheme, which constructs perturbations according to a particular nullspace condition, allowing them to be invisible (to first order in the step-size) to the network centroid, while preserving privacy guarantees. The analysis allows for general nonconvex loss functions, and is hence applicable to a large number of machine learning and signal processing problems, including deep learning.
\end{abstract}
\section{Introduction and Related Works}\label{sec:intro}
\noindent We consider a collection of \( K \) agents, where each agent \( k \) is equipped with a local loss function:
\begin{equation}
  J_k(w) \triangleq \E Q(w; \x_k)
\end{equation}
The agents are interested in pursuing a minimizer to the aggregate optimization problem:
\begin{equation}\label{eq:aggregate_problem}
  \min_{w} J(w) \triangleq \min_{w} \frac{1}{K} \sum_{k=1}^K J_k(w)
\end{equation}
While the minimizer of~\eqref{eq:aggregate_problem} can be pursued by a variety of decentralized strategies, we focus here on the Adapt-Then-Combine (ATC) diffusion strategy due to its enhanced performance in adaptive scenarios~\cite{Sayed14}:
\begin{align}
  \boldsymbol{\phi}_{k, i} =&\:  \w_{k, i-1} - \mu {\nabla Q}_k(\w_{k, i-1}; \x_{k, i}) \label{eq:adapt} \\
  \boldsymbol{w}_{k, i} =&\: \sum_{\ell \in \mathcal{N}_k} a_{\ell k} \boldsymbol{\phi}_{l, i}\label{eq:combine}
\end{align}
The diffusion strategy~\eqref{eq:adapt}--\eqref{eq:combine} has strong performance guarantees in both the (strongly) convex~\cite{Chen15transient, Chen15performance, Sayed14} and non-convex~\cite{Vlaski19nonconvexP1, Vlaski19nonconvexP2} settings. For~\eqref{eq:adapt}--\eqref{eq:combine} to minimize~\eqref{eq:aggregate_problem}, we will assume the combination weights to be symmetric and stochastic, i.e.:
\begin{equation}\label{eq:symmetry}
  a_{\ell k} = a_{k \ell}, \ \ \ \sum_{\ell = 1}^K a_{\ell k} = 1, \ \ \ a_{\ell k} \ge 0
\end{equation}
When agents are concerned about privacy, they may be hesitant to share their raw, intermediate estimates \( \boldsymbol{\phi}_{k, i} \), since they can contain significant information about their locally observed data \( \x_{k, i} \) through its evolution via the gradient. To see that this is the case, consider the least-squares loss \( Q_k(w; \boldsymbol{h}, \boldsymbol{\gamma}) \triangleq \|\boldsymbol{\gamma}-\boldsymbol{h}^{\mathsf{T}} w^o\|^2\), with stochastic gradient:
\begin{equation}
  \nabla Q(w; \boldsymbol{h}, \boldsymbol{\gamma} ) = \boldsymbol{h} \left( \boldsymbol{\gamma} - \boldsymbol{h}^{\mathsf{T}} w \right) \propto \boldsymbol{h}
\end{equation}
In other words, the stochastic gradient \( \nabla Q(w; \boldsymbol{h}, \boldsymbol{\gamma} ) \) is proportional to the raw feature \( \boldsymbol{h} \), and hence observation of iterates, which evolve according to \( \nabla Q(w; \boldsymbol{h}, \boldsymbol{\gamma} ) \) allows for the inference of \( \boldsymbol{h} \). A common strategy to ensure privacy in recursive algorithms is to perturb intermediate estimates before sharing them, resulting in~\cite{Huang14}:
\begin{align}
  \boldsymbol{\phi}_{k, i} =&\: \w_{k, i-1} - \mu {\nabla Q}_k(\w_{k, i-1}; \x_{k, i}) \label{eq:adapt_privately} \\
  \boldsymbol{\psi}_{k, i} =&\: \boldsymbol{\phi}_{k, i-1} + \boldsymbol{q}_{k, i} \label{eq:privatize} \\
  \boldsymbol{w}_{k, i} =&\: \sum_{\ell \in \mathcal{N}_k} a_{\ell k} \boldsymbol{\psi}_{l, i}\label{eq:combine_privately}
\end{align}
The added perturbation \( \boldsymbol{q}_{k, i} \) is typically chosen to follow some zero-mean Gaussian or Laplacian distribution and essentially masks the gradient \( {\nabla Q}_k(\w_{k, i-1}; \x_{k, i}) \), which contains information about the data \( \x_{k, i} \). This results in rigorous privacy guarantees (quantified by differential privacy~\cite{Dwork14}), but comes at a cost, namely non-negligible degradation in performance. To see why this is the case, we introduce the gradient noise:
\begin{equation}
  \s_{k, i} \triangleq {\nabla Q}_k(\w_{k, i-1}; \x_{k, i}) - {\nabla J}_k(\w_{k, i-1})
\end{equation}
Under this definition, recursion~\eqref{eq:adapt_privately}--\eqref{eq:combine_privately} can be written equivalently as:
\begin{align}
  \boldsymbol{\phi}_{k, i} =&\: \w_{k, i-1} - \mu {\nabla J}_k(\w_{k, i-1}) - \mu \s_{k, i} \label{eq:adapt_privately_grad_noise} \\
  \boldsymbol{\psi}_{k, i} =&\: \boldsymbol{\phi}_{k, i-1} + \boldsymbol{q}_{k, i} \\
  \boldsymbol{w}_{k, i} =&\: \sum_{\ell \in \mathcal{N}_k} a_{\ell k} \boldsymbol{\psi}_{l, i}\label{eq:combine_privately_grad_noise}
\end{align}
Inspection of~\eqref{eq:adapt_privately_grad_noise}--\eqref{eq:combine_privately_grad_noise} shows that the effect of privatizing the gradient \( {\nabla Q}_k(\w_{k, i-1}; \x_{k, i}) \) is amplification of the gradient noise term \( \mu \s_{k, i} \) by an additive term \( \boldsymbol{q}_{k, i} \).

\subsection{Related Works}
\noindent Solutions to the aggregate optimization problem~\eqref{eq:aggregate_problem} can be pursued by a variety of decentralized algorithms, including primal~\cite{Nedic09, Chen15transient, Chen15performance, Sayed14} and primal-dual~\cite{Shi15, Paolo16, Yuan18, Xin19, Jakovetic20} methods.

The notion of \( \epsilon \)-differential privacy as a means of quantifying the privacy loss encountered by sharing functions of private data is due to~\cite{Dwork06, Dwork14}, as is the Laplace mechanism, which ensures \( \epsilon \)-differential privacy by perturbing the output of the function by Laplacian noise, where the power of the perturbation is calibrated to the sensitivity of the function and the desired privacy level \( \epsilon \).

In the context of centralized optimization by means of recursive algorithms, differential privacy has been applied to gradient descent~\cite{Rajkumar12, Song13, Lee18privacy}, deep learning~\cite{Abadi16}, as well as federated learning~\cite{Geyer17,Wei20}. The decentralized setting considered in this work is studied in~\cite{Huang14, Zhang17, Li18, Showkatbakhsh19, Hou19}, where independent and identically distributed perturbations are added at each agent as in~\eqref{eq:privatize} and differential privacy is established.

Similarly to these related works, our scheme is based on stochastic gradient descent, and employs perturbations to achieve privacy. In contrast to prior works, however, locally generated perturbations at each agent will be tuned to the local graph topology, ensuring that the effect on the evolution of the network centroid is minimized, while preserving privacy guarantees. The authors in~\cite{Guo20} present a ``topology-aware'' perturbation scheme, where noise powers are tuned to the local connectivity of agents. We, on the other hand, will be constructing the actual realizations, rather than perturbation powers, to match the graph topology.

\section{Diffusion with Graph-Homomorphic Perturbations}
\noindent We generalize the scheme~\eqref{eq:adapt_privately_grad_noise}--\eqref{eq:combine_privately_grad_noise}, and allow agent \( \ell \) to send different perturbation vectors \( \boldsymbol{q}_{\ell k, i} \) to different neighbors \( k \), resulting in:
\begin{align}
  \boldsymbol{\phi}_{k, i} =&\:  \w_{k, i-1} - \mu {\nabla J}_k(\w_{k, i-1}) - \mu \s_{k, i}  \label{eq:proposed_adapt} \\
  \boldsymbol{\psi}_{k \ell, i} =&\: \boldsymbol{\phi}_{k, i} + \boldsymbol{q}_{k \ell, i} \label{eq:proposed_privatize} \\
  \boldsymbol{w}_{k, i} =&\: \sum_{\ell \in \mathcal{N}_k} a_{\ell k} \boldsymbol{\psi}_{\ell k, i} \label{eq:proposed_combine}
\end{align}
Our objective is to exploit this additional degree of freedom to construct the perturbations \( \boldsymbol{q}_{\ell k, i} \) in a manner that protects agent \( \ell \) from agent \( k \), but minimizes the negative effect on the network as a whole. Previous studies on the dynamics of the diffusion recursion without privacy guarantees have shown that the local dynamics of each agent closely track those of a network centroid after sufficient iterations, both in the convex~\cite{Chen15transient, Chen15performance} and nonconvex~\cite{Vlaski19nonconvexP1, Vlaski19nonconvexP2} setting. From~\eqref{eq:proposed_combine}, we find for the network centroid:
\begin{align}
  \w_{c, i} \triangleq&\: \frac{1}{K} \sum_{k=1}^K \w_{k, i} \notag \\
  \stackrel{\eqref{eq:proposed_combine}}{=}&\: \frac{1}{K} \sum_{k=1}^K \sum_{\ell = 1}^{K} a_{\ell k} \boldsymbol{\phi}_{\ell, i} + \frac{1}{K} \sum_{k=1}^K \sum_{\ell =1}^K a_{\ell k} \boldsymbol{q}_{\ell k, i} \notag \\
  =&\: \frac{1}{K} \sum_{\ell = 1}^{K} \left( \sum_{k=1}^K a_{\ell k}\right) \boldsymbol{\phi}_{\ell, i} + \frac{1}{K} \sum_{\ell =1}^K \sum_{k=1}^K a_{\ell k} \boldsymbol{q}_{\ell k, i} \notag \\
  \stackrel{\eqref{eq:symmetry}}{=}&\: \frac{1}{K} \sum_{\ell = 1}^{K} \boldsymbol{\phi}_{\ell, i} + \frac{1}{K} \sum_{\ell =1}^K \sum_{k=1}^K a_{\ell k} \boldsymbol{q}_{\ell k, i} \notag \\
  \stackrel{\eqref{eq:adapt_privately}}{=}&\: \w_{c, i-1} - \frac{\mu}{K} \sum_{\ell = 1}^{K} \nabla Q(\w_{k, i-1}; \x_{k, i}) \notag \\
  &\:+ \frac{1}{K} \sum_{\ell =1}^K \sum_{k=1}^K a_{\ell k} \boldsymbol{q}_{\ell k, i} \label{eq:centroid_evolution}
\end{align}
We observe that the network centroid \( \w_{c, i} \) evolves similarly to a stochastic gradient update on the aggregate loss~\eqref{eq:aggregate_problem}, perturbed by the sample mean of the weighted privacy terms \( a_{\ell k} \boldsymbol{q}_{\ell k, i} \). The key question then is whether it is possible to construct \( \boldsymbol{q}_{\ell k, i} \) in an uncoordinated manner, such that:
\begin{equation}\label{eq:desired}
  \frac{1}{K} \sum_{\ell=1}^K \sum_{k =1}^{K} a_{\ell k} \boldsymbol{q}_{\ell k, i} \stackrel{\mathrm{desired}}{=} 0
\end{equation}
while preserving the privacy of all agents. If this were the case, the evolution of the network centroid would be largely unaffected by the privacy perturbations. We say ``largely unaffected'', since the gradients \( \nabla Q(\w_{k, i-1}; \x_{k, i}) \) are evaluated at \( \w_{k, i-1} \), rather than \( \w_{c, i-1} \) and hence indirectly affected by the privacy perturbations. As such, a more detailed performance analysis is necessary, which we conduct further below. The key take-away from the analysis will be that, despite the fact that the perturbations added in~\eqref{eq:proposed_privatize} are independent of the step-size, ensuring~\eqref{eq:desired} modulates the effect of the privacy perturbations on the evolution of the centroid by a factor of the step-size \( \mu \), allowing for increasing levels of privacy perturbation as the step-size decreases. Since perturbations satisfying~\eqref{eq:desired} have a reduced effect on the evolution of the network centroid under the adjacency matrix \( A \), we will refer to them as ``graph-homomorphic''. 
\begin{definition}[\textbf{Graph-Homomorphic Perturbations}]
  A set of perturbations \( \boldsymbol{q}_{\ell k, i} \) is homomorphic for the the graph defined by the adjacency matrix \( A \triangleq [a_{\ell k}] \) if it holds with probability one that:
  \begin{equation}\label{eq:nullspace_cond}
    \frac{1}{K} \sum_{\ell=1}^K \sum_{k =1}^{K} a_{\ell k} \boldsymbol{q}_{\ell k, i} = 0
  \end{equation}
\end{definition}
\noindent While other constructions are possible, we present here a simple construction, which can be implemented locally and independently at every agent \( k \).
\begin{lemma}[\textbf{Constructing Graph-Homomorphic Perturbations}]
  Let each agent \( \ell \) sample independently from the Laplace distribution \( \boldsymbol{v}_{\ell, i} \sim \mathrm{Lap}\left(0, b_v \right) \) with variance \( \sigma_v^2 = 2 b_v^2 \). Then, the construction:
  \begin{align}\label{eq:homomorphic_construction}
    \boldsymbol{q}_{\ell k, i} = \begin{cases} \boldsymbol{v}_{\ell, i},\ \ &\mathrm{if}\ k \in \mathcal{N}_{\ell} \ \mathrm{and}\ k \neq \ell,\\ -\frac{1-a_{\ell \ell}}{a_{\ell \ell}} \boldsymbol{v}_{\ell, i}, \ \ &\mathrm{if}\ k = \ell. \end{cases}
  \end{align}
  is homomorphic for the graph described by the symmetric adjacency matrix \( A = A^{\mathsf{T}} \).
\end{lemma}
\begin{proof}
  The result can be verified immediately by substitution.
\end{proof}

\section{Analysis}
\subsection{Modeling Conditions}
\noindent We make the following common assumptions to facilitate the performance and privacy analysis.
\begin{assumption}[\textbf{Adjacency matrix}]\label{as:graph}
  The adjacency matrix \( A \triangleq \left[ a_{\ell k} \right] \) is symmetric and doubly-stochastic, i.e.:
  \begin{equation}
    a_{\ell k} = a_{k \ell},\ \sum_{\ell \in \mathcal{N}_k} a_{\ell k} = 1, \ a_{\ell k} = 0\ \forall\ \ell \notin \mathcal{N}_k  \end{equation}
  Furthermore, the graph described by \( A \) is connected, ensuring that:
  \begin{equation}
    \lambda_2 \triangleq \rho\left( A - \frac{1}{K} \mathds{1}\mathds{1}^{\mathsf{T}} \right) < 1
  \end{equation}
\end{assumption}
\begin{assumption}[\textbf{Smoothness}]\label{as:smoothness}
  The risk functions \( Q(\cdot; \boldsymbol{x}_k) \) have uniformly Lipschitz gradients, i.e. for all \( w_1, w_2 \), and with probability one:
  \begin{equation}
    \|\nabla Q(w_1; \boldsymbol{x}_k) - \nabla Q(w_2; \boldsymbol{x}_k)\| \le \delta \|w_1 - w_2\|
  \end{equation}
  Additionally, we impose a bound on the norm of the stochastic gradient:
  \begin{equation}
    \|\nabla Q(w; \boldsymbol{x}_k) \| \le G
  \end{equation}
\end{assumption}

\subsection{Privacy Analysis}
\noindent We now proceed to quantify the privacy loss encountered by a particular agent, when deciding to participate in the learning protocol. To quantify privacy precisely, we will employ the notion of \( \epsilon \)-differential privacy~\cite{Dwork14}. For simplicity of exposition, and without loss of generality, we will focus on establishing a privacy guarantee for agent \( 1 \). By symmetry, the same argument applies to all other agents as well.

To this end, consider an alternative scenario, where agent \( 1 \) has decided not to volunteer its private information for the diffusion of information, and its data \( \boldsymbol{x}_1 \) is replaced by some other data \( \boldsymbol{x}_1' \), following a different distribution. In this setting, implementing~\eqref{eq:proposed_adapt}--\eqref{eq:proposed_combine}, would naturally result in a different learning trajectory \( \w_{k, i}' \) at every agent \( k \), since the data \( \boldsymbol{x}_1' \) propagates through \( \nabla Q(\w_{1, i}', \boldsymbol{x}_1') \) and the diffusion of estimates through the entire network. We first quantify the sensitivity of the evolution of the algorithm~\eqref{eq:proposed_adapt}--\eqref{eq:proposed_combine}, a quantity that determines the amount of perturbation necessary to mask any particular agent~\cite{Dwork14, Huang14}.
\begin{lemma}[\textbf{Sensitivity of the diffusion algorithm}]\label{LEM:SENSITIVITY}
  The distance between the trajectories \( \w_{k, i} \) and \( \w_{k, i}' \) is bounded with probability one by:
  \begin{equation}\label{eq:sensitivity_bound}
    \Delta(i) \triangleq \max_k \| \w_{k, i} - \w_{k, i}' \| \le \mu 2 G i
  \end{equation}
\end{lemma}
\begin{proof}
  Omitted due to space limitations.
\end{proof}
\begin{definition}[\textbf{\( \epsilon \)-differential privacy}]
  We say that the diffusion recursion~\eqref{eq:proposed_adapt}--\eqref{eq:proposed_combine} is \( \epsilon(i) \)-differentially private for agent \( 1 \) at time \( i \) if:
  \begin{equation}\label{eq:differential_privacy}
    \frac{f\left( \left\{\left\{\boldsymbol{\psi}_{1\ell, n} \right\}_{\ell \neq 1 \in \mathcal{N}_1} \right\}_{n=0}^i \right)}{f\left( \left\{\left\{\boldsymbol{\psi}'_{1\ell, n} \right\}_{\ell \neq 1 \in \mathcal{N}_1} \right\}_{n=0}^i \right)} \le e^{\epsilon(i)}
  \end{equation}
  where \( f(\cdot) \) denotes the probability density function and \( \left\{\left\{\boldsymbol{\psi}_{1\ell, n} \right\}_{\ell \neq 1 \in \mathcal{N}_1} \right\}_{n=0}^i \) collects all quantities transmitted by agent \( 1 \) to any of its neighbors during the operation of the algorithm, while excluding its local iterates \( \boldsymbol{\psi}_{11, n} \), which are kept private.
\end{definition}
\noindent In light of the fact that \( e^{\epsilon(i)} \approx 1 - \epsilon(i) \) for small \( \epsilon(i) \), relation~\eqref{eq:differential_privacy} ensures that the distribution of estimates shared by agent \( 1 \) is close to unaffected (for small \( \epsilon(i) \)), whether agent \( 1 \) uses its own private data \( \boldsymbol{x}_{1} \), or a proxy \( \x'_1 \), and as such little can be inferred about \( \boldsymbol{x}_{1} \) by observing messages shared by agent \(1\).
\begin{theorem}[\textbf{Privacy cost of the diffusion algorithm}]\label{TH:PRIVACY}
  Suppose~\eqref{eq:proposed_adapt}--\eqref{eq:proposed_combine} employs homomorphic perturbations constructed as in~\eqref{eq:homomorphic_construction}. Then, at time \( i \), algorithm~\eqref{eq:proposed_adapt}--\eqref{eq:proposed_combine} is \( \epsilon(i)\)-differentially private according to~\eqref{eq:differential_privacy}, with:
  \begin{equation}\label{eq:privacy_loss}
    \epsilon(i) = \mu \frac{G (i^2 + i)}{b_v}
  \end{equation}
\end{theorem}
\begin{proof}
  Omitted due to space limitations.
\end{proof}

\subsection{Performance Analysis}
\noindent In order to quantify the impact of the privacy perturbations on the performance of the algorithm, we now conduct a performance analysis in the presence of perturbations. Following the arguments in~\cite{Vlaski19nonconvexP1} for analyzing the dynamics of the unperturbed recursion~\eqref{eq:adapt}--\eqref{eq:combine} in nonconvex environments, we begin by establishing that the collection of iterates \( \{\w_{k, i}\}_{k=1}^K \) continue to cluster around the network centroid \( \w_{c, i} \).
\begin{lemma}[\textbf{Network Disagreement}]\label{LEM:DISAGREEMENT}
  Suppose the collection of agents \( \{\w_{k, i}\}_{k=1}^K \) is initialized at a common, non-informative location, say \( \w_{k, 0} = \mathrm{col}\left\{ 0, \ldots, 0 \right\} \) for all \( k \). Then, the deviation from the centroid is bounded for all \( i \ge 0 \) as:
  \begin{equation}
    \frac{1}{K} \sum_{k=1}^K \E \left\|\w_{k, i} - \w_{c, i}\right\|^2 \le \mu^2 \frac{\lambda_2^2}{(1-\lambda_2)^2} G^2 + b_v^2 \frac{2 \overline{a}}{1-\lambda_2}
  \end{equation}
  where:
  \begin{align}
    \overline{a} \triangleq \max_k \left\{ (1-a_{kk}) + \frac{(1-a_{kk})^2}{a_{kk}^2} \right\}
  \end{align}
\end{lemma}
\begin{proof}
  Omitted due to space limitations.
\end{proof}
\noindent Relative to performance expressions for non-private decentralized gradient descent, we observe that the privacy perturbations account for an additional deviation on the order of \( O(b_v^2) \). Nevertheless, relation~\eqref{eq:desired} under~\eqref{eq:homomorphic_construction} allows us to establish an improved descent relation.
\begin{theorem}[\textbf{Descent relation}]\label{TH:PERFORMANCE}
  Under Assumptions~\ref{as:graph}--\ref{as:smoothness}, and for homomorphic perturbations constructed as in~\eqref{eq:homomorphic_construction}, the network centroid descends along the loss~\eqref{eq:aggregate_problem} as:
  \begin{align}\label{eq:descent_relation}
    \E J(\w_{c, i}) \le&\: \E J(\w_{c, i-1}) - \frac{\mu}{2}(1-2\mu\delta) \E \|\nabla J(\w_{c, i-1})\|^2 \notag \\
    &\:+ \frac{\mu}{2}(1+2\delta\mu)  b_v^2 \frac{2 \delta^2 \overline{a}}{1-\lambda_2} + \mu^2 2 \delta G^2 + O(\mu^3)
  \end{align}
\end{theorem}
\begin{proof}
  Omitted due to space limitations.
\end{proof}
\noindent Examination of~\eqref{eq:descent_relation} reveals that, despite the fact that the amount of perturbations added in~\eqref{eq:proposed_privatize} is independent of the step-size, their negative effect on the ability of the network centroid to descend along the aggregate loss \( J(w) \) is multiplied by \( \mu \), and hence decays with the step-size.
\begin{corollary}[\textbf{Convergence to stationary points}]
  Suppose \( J(w) \ge J^o \). Then, under Assumptions~\ref{as:graph}--\ref{as:smoothness}, and for homomorphic perturbations constructed as in~\eqref{eq:homomorphic_construction}, we have:
  \begin{align}
    \frac{1}{i} \sum_{n=0}^{i-1} \E \|\nabla J(\w_{c, n}\|^2 \le O\left(\frac{1}{\mu i}\right) + O(b_v^2) + O(\mu G^2)
  \end{align}
\end{corollary}
\begin{proof}
  The result follows after rearranging~\eqref{eq:descent_relation} and telescoping.
\end{proof}

\section{Numerical Results}
We verify the analytical results in the context of decentralized logistic regression for binary classification. Given class labels \( \boldsymbol{\gamma} \in \left\{ +1, -1 \right\} \), we construct feature vectors \( \boldsymbol{h} \) to be conditionally Gaussian, with means \( \mu_{+1} \) and \( \mu_{-1} \) respectively, i.e., \( \boldsymbol{h} \in \mathds{R}^M \) with \( f(\boldsymbol{h} | \boldsymbol{\gamma} = \gamma) = \mathcal{N}\left( \mu_{\gamma}, \sigma_h^2 \right) \). Each agent \( k \) is equipped with a local logistic loss function of the form:
\begin{equation}
  J_k(w) \triangleq \E \ln\left( 1 + e^{-\boldsymbol{\gamma} \boldsymbol{h}^{\T} w} \right) + \frac{\rho}{2} \|w\|^2
\end{equation}
We compare the performance of the ordinary diffusion recursion~\eqref{eq:adapt}--\eqref{eq:combine} with the privatized recursion~\eqref{eq:adapt_privately}--\eqref{eq:combine_privately} and the proposed scheme~\eqref{eq:proposed_adapt}--\eqref{eq:proposed_combine}, constructed according to~\eqref{eq:homomorphic_construction}. The resulting performance is illustrated in Fig.~\ref{fig:performance}.
\begin{figure}
  \includegraphics[width=.9\linewidth]{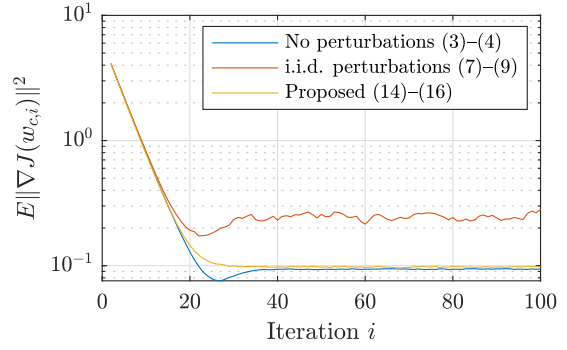}
  \centering
  \caption{Performance comparison with \( M = 5 \), \( K = 20 \), \( \mu = 1\), \( \rho = 0.1\), \( \sigma_h^2 = 1\), \( \sigma_p^2 = 2 \).}\label{fig:performance}
\end{figure}
We observe that the proposed perturbation scheme approximately matches the performance of the non-private diffusion implementation, while outperforming the implementation with i.i.d. perturbations, despite employing the same perturbation powers \( \sigma_p^2 \).

\section{Conclusion}
We have proposed a new perturbation scheme for differentially private decentralized stochastic optimization, where the perturbations are constructed at each agent to match the local graph topology. The resulting perturbations are invisible to the network centroid under the diffusion operation, while preserving \( \epsilon \)-differential privacy, and hence termed graph-homomorphic (for a particular topology). Analytical and numerical results show that the construction reduces the negative effect of privacy perturbations, while preserving differential privacy.

\bibliographystyle{IEEEbib}
{\bibliography{main}}
\balance

\end{document}

%% file: custom/operators.tex
\usepackage{amsthm}

\DeclareMathOperator{\T}{\mathsf{T}}
\DeclareMathOperator{\E}{\mathds{E}}

\DeclareMathOperator{\w}{\boldsymbol{w}}
\DeclareMathOperator{\x}{\boldsymbol{x}}
\DeclareMathOperator{\s}{\boldsymbol{s}}
\DeclareMathOperator{\g}{\boldsymbol{g}}

%% file: custom/environments.tex
\usepackage{amsthm}

\theoremstyle{plain}

\newtheorem{definition}{Definition}
\newtheorem{assumption}{Assumption}
\newtheorem{theorem}{Theorem}
\newtheorem{corollary}{Corollary}
\newtheorem{lemma}{Lemma}

%% file: main.bbl
\begin{thebibliography}{10}

\bibitem{Sayed14}
A.~H. Sayed,
\newblock ``{Adaptation, learning, and optimization over networks},''
\newblock {\em Foundations and Trends in Machine Learning}, vol. 7, no. 4-5,
  pp. 311--801, July 2014.

\bibitem{Chen15transient}
J.~Chen and A.~H. Sayed,
\newblock ``On the learning behavior of adaptive networks - {Part I}: Transient
  analysis,''
\newblock {\em IEEE Transactions on Information Theory}, vol. 61, no. 6, pp.
  3487--3517, June 2015.

\bibitem{Chen15performance}
J.~Chen and A.~H. Sayed,
\newblock ``On the learning behavior of adaptive networks -- {Part II}:
  Performance analysis,''
\newblock {\em IEEE Transactions on Information Theory}, vol. 61, no. 6, pp.
  3518--3548, June 2015.

\bibitem{Vlaski19nonconvexP1}
{S. Vlaski and A. H. Sayed},
\newblock ``{Distributed learning in non-convex environments -- Part I:
  Agreement at a Linear rate},''
\newblock {\em available as arXiv:1907.01848}, July 2019.

\bibitem{Vlaski19nonconvexP2}
{S. Vlaski and A. H. Sayed},
\newblock ``{Distributed learning in non-convex environments -- Part II:
  Polynomial escape from saddle-points},''
\newblock {\em available as arXiv:1907.01849}, July 2019.

\bibitem{Huang14}
Z.~Huang, S.~Mitra, and N.~Vaidya,
\newblock ``Differentially private distributed optimization,''
\newblock in {\em Proc. International Conference on Distributed Computing and
  Networking}, Goa, India, Jan. 2015.

\bibitem{Dwork14}
C.~Dwork and A.~Roth,
\newblock ``The algorithmic foundations of differential privacy,''
\newblock {\em Found. Trends Theor. Comput. Sci.}, vol. 9, no. 3–4, pp.
  211–407, Aug. 2014.

\bibitem{Nedic09}
A.~Nedic and A.~Ozdaglar,
\newblock ``{Distributed subgradient methods for multi-agent optimization},''
\newblock {\em IEEE Trans. Automatic Control}, vol. 54, no. 1, pp. 48--61, Jan
  2009.

\bibitem{Shi15}
W.~Shi, Q.~Ling, G.~Wu, and W.~Yin,
\newblock ``Extra: An exact first-order algorithm for decentralized consensus
  optimization,''
\newblock {\em SIAM Journal on Optimization}, vol. 25, no. 2, pp. 944--966,
  2015.

\bibitem{Paolo16}
P.~{Di Lorenzo} and G.~{Scutari},
\newblock ``Next: In-network nonconvex optimization,''
\newblock {\em IEEE Transactions on Signal and Information Processing over
  Networks}, vol. 2, no. 2, pp. 120--136, 2016.

\bibitem{Yuan18}
K.~Yuan, B.~Ying, X.~Zhao, and A.~H. Sayed,
\newblock ``Exact diffusion for distributed optimization and learning -- {Part
  II}: Convergence analysis,''
\newblock {\em IEEE Transactions on Signal Processing}, vol. 67, no. 3, pp.
  724--739, Feb 2019.

\bibitem{Xin19}
R.~{Xin}, A.~K. {Sahu}, U.~A. {Khan}, and S.~{Kar},
\newblock ``Distributed stochastic optimization with gradient tracking over
  strongly-connected networks,''
\newblock in {\em Proc. IEEE 58th Conference on Decision and Control (CDC)},
  2019, pp. 8353--8358.

\bibitem{Jakovetic20}
D.~{Jakovetić}, D.~{Bajović}, J.~{Xavier}, and J.~M.~F. {Moura},
\newblock ``Primal-dual methods for large-scale and distributed convex
  optimization and data analytics,''
\newblock {\em Proceedings of the IEEE}, pp. 1--16, 2020.

\bibitem{Dwork06}
C.~Dwork, F.~McSherry, K.~Nissim, and A.~Smith,
\newblock ``Calibrating noise to sensitivity in private data analysis,''
\newblock in {\em Theory of Cryptography}, Berlin, Heidelberg, 2006, pp.
  265--284, Springer Berlin Heidelberg.

\bibitem{Rajkumar12}
A.~Rajkumar and S.~Agarwal,
\newblock ``A differentially private stochastic gradient descent algorithm for
  multiparty classification,''
\newblock in {\em Proc. Machine Learning Research}, La Palma, Canary Islands,
  Apr 2012, pp. 933--941.

\bibitem{Song13}
S.~{Song}, K.~{Chaudhuri}, and A.~D. {Sarwate},
\newblock ``Stochastic gradient descent with differentially private updates,''
\newblock in {\em Proc. IEEE Global Conference on Signal and Information
  Processing}, 2013, pp. 245--248.

\bibitem{Lee18privacy}
J.~Lee and D.~Kifer,
\newblock ``Concentrated differentially private gradient descent with adaptive
  per-iteration privacy budget,''
\newblock in {\em Proc. of ACM SIGKDD International Conference on Knowledge
  Discovery and Data Mining}, July 2018, pp. 1656--1665.

\bibitem{Abadi16}
M.~Abadi, A.~Chu, I.~Goodfellow, H.~B. McMahan, I.~Mironov, K.~Talwar, and
  L.~Zhang,
\newblock ``Deep learning with differential privacy,''
\newblock in {\em Proc. ACM SIGSAC Conference on Computer and Communications
  Security}, Vienna, Austria, 2016, p. 308–318.

\bibitem{Geyer17}
R.~C. Geyer, T.~Klein, and M.~Nabi,
\newblock ``Differentially private federated learning: A client level
  perspective,''
\newblock {\em available as arXiv:1712.07557}, Dec 2017.

\bibitem{Wei20}
K.~{Wei}, J.~{Li}, M.~{Ding}, C.~{Ma}, H.~H. {Yang}, F.~{Farokhi}, S.~{Jin},
  T.~Q.~S. {Quek}, and H.~{Vincent Poor},
\newblock ``Federated learning with differential privacy: Algorithms and
  performance analysis,''
\newblock {\em IEEE Transactions on Information Forensics and Security}, vol.
  15, pp. 3454--3469, 2020.

\bibitem{Zhang17}
T.~{Zhang} and Q.~{Zhu},
\newblock ``Dynamic differential privacy for admm-based distributed
  classification learning,''
\newblock {\em IEEE Transactions on Information Forensics and Security}, vol.
  12, no. 1, pp. 172--187, 2017.

\bibitem{Li18}
C.~{Li}, P.~{Zhou}, L.~{Xiong}, Q.~{Wang}, and T.~{Wang},
\newblock ``Differentially private distributed online learning,''
\newblock {\em IEEE Transactions on Knowledge and Data Engineering}, vol. 30,
  no. 8, pp. 1440--1453, 2018.

\bibitem{Showkatbakhsh19}
M.~Showkatbakhsh, C.~Karakus, and S.~Diggavi,
\newblock ``Differentially private consensus-based distributed optimization,''
\newblock {\em available as arXiv:1903.07792}, March 2019.

\bibitem{Hou19}
M.~{Hou}, D.~{Li}, X.~{Wu}, and X.~{Shen},
\newblock ``Differential privacy of online distributed optimization under
  adversarial nodes,''
\newblock in {\em 2019 Chinese Control Conference (CCC)}, 2019, pp. 2172--2177.

\bibitem{Guo20}
T.~Xiang Y.~Liu S.~Guo, T.~Zhang,
\newblock ``Differentially private decentralized learning,''
\newblock {\em available as arXiv:2006.07817}, June 2020.

\end{thebibliography}
